\documentclass[runningheads]{llncs}
\usepackage[T1]{fontenc}
\usepackage{graphicx}
\usepackage[colorlinks=true,
            linkcolor=blue,
            citecolor=blue,
            urlcolor=blue
]{hyperref}
\usepackage{amsmath}
\usepackage{amsfonts}
\usepackage{amssymb}
\usepackage{enumitem}
\usepackage[capitalize]{cleveref}
\usepackage{tikz}
\usetikzlibrary{shapes.arrows, fadings}
\usepackage{subcaption}
\usepackage[misc,geometry]{ifsym}

\usepackage{color}

\usepackage{xcolor}

\definecolor{bittersweet}{rgb}{1.0, 0.44, 0.37}
\newcommand{\sk}{}
\definecolor{thetruth}{rgb}{0,0.701,0.117}

\definecolor{fourierblue}{rgb}{0.78, 0.84, 0.93}

\newcommand{\R}{\mathbb{R}}

\newcommand{\domain}{\Omega}
\newcommand{\Nem}{Nemytskii}
\newcommand{\abs}[1]{\left|#1\right|}

\makeatletter
\let\blx@rerun@biber\relax
\makeatother

\begin{document}
\title{Resolution-Invariant Image Classification based on Fourier Neural Operators\thanks{This work was supported by the European Union’s Horizon 2020 programme
, Marie Sk\l odowska-Curie grant agreement No. 777826.
TR and MB acknowledge the support of the BMBF, grant agreement No. 05M2020. 
SK and MB acknowledge the support of the DFG, project BU 2327/19-1. This work was carried out while MB was with the FAU Erlangen-Nürnberg.} }
\titlerunning{Resolution-Invariant Image Classification based on FNOs}
\author{%
Samira Kabri\inst{1}$^{(\text{\Letter})}$
\and%
Tim Roith\inst{1}
\and%
Daniel Tenbrinck\inst{1}
\and%
Martin Burger\inst{2,3}
}
\institute{Friedrich-Alexander-Universität Erlangen-N\"urnberg, 91058 Erlangen, Germany \and Deutsches Elektronen-Synchrotron, 22607 Hamburg, Germany \and Universit\"at Hamburg, Fachbereich Mathematik, 20146 Hamburg, Germany \\ \Letter\,\email{samira.kabri@fau.de} 
}
\authorrunning{S. Kabri et al.}
\titlerunning{Resolution-Invariant Image Classification based on FNOs}
\maketitle              
\begin{abstract}
In this paper we investigate the use of Fourier Neural Operators (FNOs) for image classification in comparison to standard Convolutional Neural Networks (CNNs). Neural operators are a discretization-invariant generalization of neural networks to approximate operators between infinite dimensional function spaces. FNOs---which are neural operators with a specific parametrization---have been applied successfully in the context of parametric PDEs. We derive the FNO architecture as an example for continuous and Fréchet-differentiable neural operators on Lebesgue spaces. We further show how CNNs can be converted into FNOs and vice versa and propose an interpolation-equivariant adaptation of the architecture.
\end{abstract}
\section{Introduction}
Neural networks, in particular CNNs, are a highly effective tool for image classification tasks. Substituting fully-connected layers by convolutional layers allows for efficient extraction of local features at different levels of detail with reasonably low complexity. However, neural networks in general are not resolution-invariant, meaning that they do not generalize well to unseen input resolutions. In addition to interpolation of inputs to the training resolution, various other approaches have been proposed to address this issue, see, e.g., \cite{cai2019convolutional,koziarski2018impact,peng2016fine}. In this work we focus on the interpretation of digital images as discretizations of functions. This allows to model the feature extractor as a mapping between infinite dimensional spaces with the help of so-called neural operators, see \cite{kovachki2021neural}. In \cref{sec:nemytskii}, we use established results on Nemytskii operators to derive conditions for well-definedness, continuity, and Fréchet-differentiability of neural operators on Lebesgue spaces. We specifically show these properties for the class of FNOs proposed in \cite{li2021fourier} as a discretization-invariant generalization of CNNs. 

The key idea of FNOs is to parametrize convolutional kernels by their Fourier coefficients, i.e., in the spectral domain. Using trainable filters in the Fourier domain to represent convolution kernels in the context of image processing with neural networks has been studied with respect to performance and robustness in recent works, see e.g., \cite{chi2020fast,rao2021global,zhou2022deep}. In \cref{sec:cnnfno} we analyze the interchangeability of CNNs and FNOs with respect to optimization, parameter complexity, and generalization to varying input resolutions. While we restrict our theoretical derivations to real-valued functions, we note that they can be naturally extended to vector-valued functions as well.  Our findings are supported by numerical experiments on the FashionMNIST \cite{xiao2017} and Birds500  \cite{pio450} data sets in \cref{sec:num}.

\begin{figure}[t]
\begin{minipage}[t]{\textwidth}
\begin{minipage}{\textwidth}
\begin{minipage}{0.01\textwidth}%
\rotatebox{90}{%
\scriptsize%
higher resolution \phantom{--------} 
original \phantom{-----------} 
lower resolution\phantom{--}}
\end{minipage}%
\hfill%
\begin{minipage}{0.95\textwidth}
\includegraphics[width=\textwidth, trim= 0cm 1cm 0cm 1cm, clip]{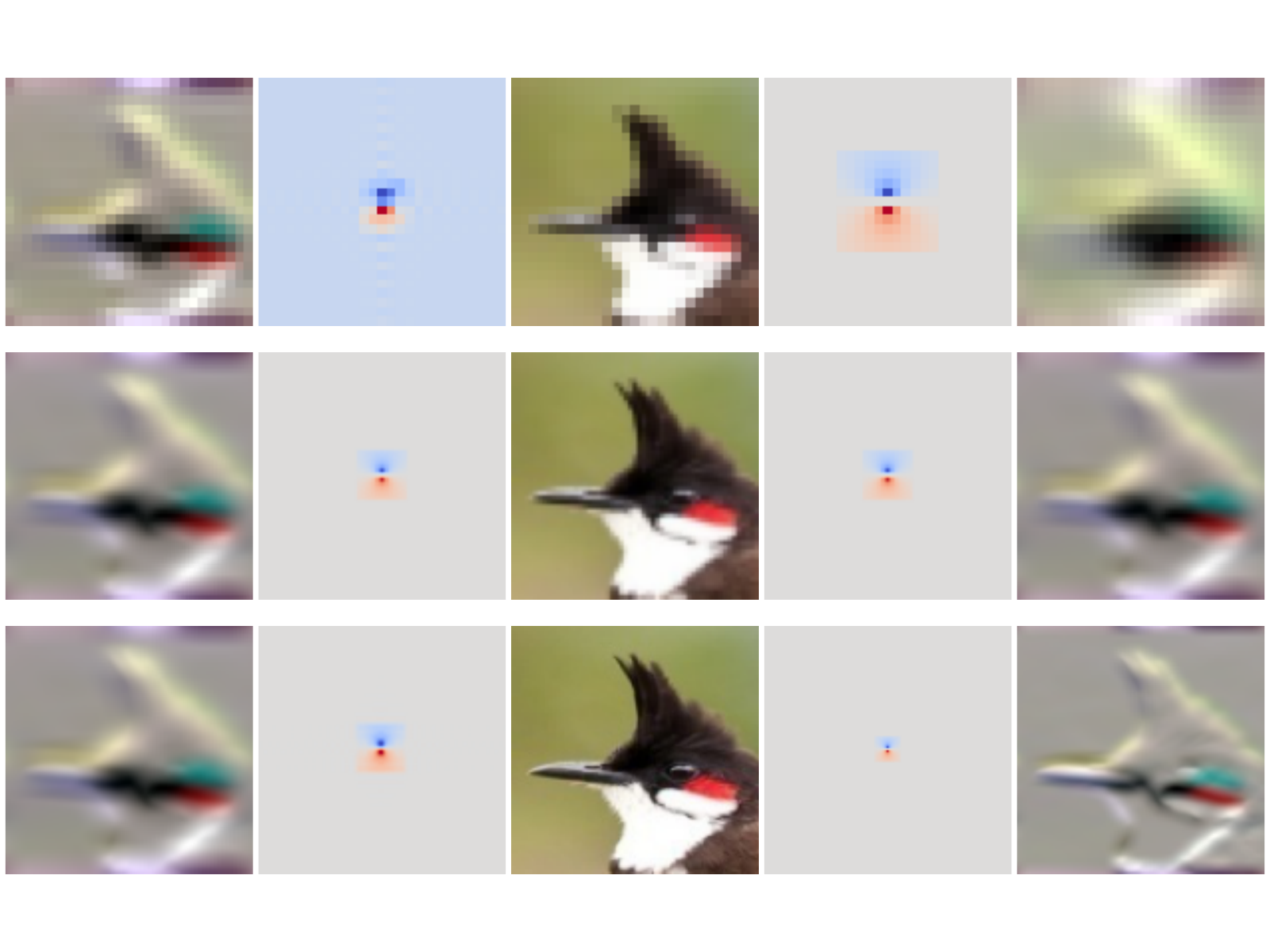}%
\end{minipage}%
\hfill%
\begin{minipage}{0.01\textwidth}
\phantom{-}
\end{minipage}
\end{minipage}
\hfill%
\begin{minipage}[t]{.5\textwidth}%
\begin{tikzpicture}[]
\node [
    draw=none,
    single arrow,
    right color=white,
    left color=fourierblue,
    text=black,
    single arrow head extend=0.2cm,
    minimum height=\textwidth-5pt,
    minimum width=.5cm,
    single arrow tip angle=70,
    shape border rotate=180
    ]{\scriptsize Convolution with 
    \textbf{spectral} zero-padding};
\end{tikzpicture}\hfill%
\end{minipage}%
\begin{minipage}[t]{.5\textwidth}%
\hfill%
\begin{tikzpicture}[]
\node [
    draw=none,
    single arrow,
    right color=fourierblue,
    left color=white,
    text=black,
    single arrow head extend=0.2cm,
    minimum height=\textwidth-5pt,
    minimum width=.5cm,
    single arrow tip angle=70,
    shape border rotate=0
    ]{\scriptsize \phantom{-}Convolution with 
    \textbf{spatial} zero-padding};
\end{tikzpicture}%
\end{minipage}%
\end{minipage}%
\caption[Effects of applying convolutional filters with different resolutions]{%
Effects of applying a convolutional filter on the same 
image\footnotemark{} 
with different resolutions. Spatial zero-padding (standard CNN-implementation) changes the relation of kernel support to image domain, while spectral zero-padding (FNO-implementation) captures comparable features for all resolutions.}
\label{fig:bulbul}
\end{figure}
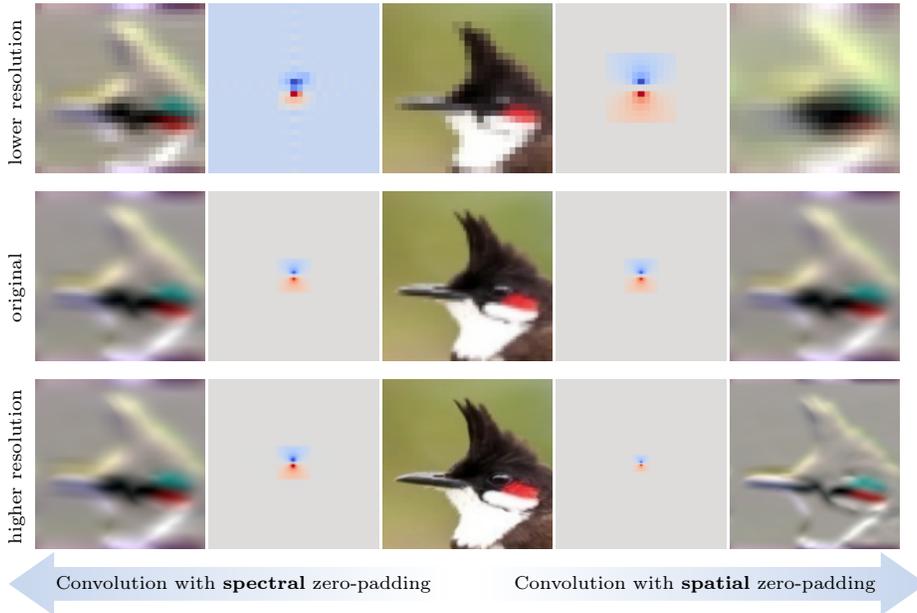
\footnotetext{This image depicts a red whiskered bulbul taken from the Birds500 dataset \cite{pio450}.}
\section{Construction of Neural Operators on Lebesgue Spaces}\label{sec:nemytskii}
\subsection{Well-definedness and Continuity}
A neural operator as defined in \cite{kovachki2021neural} is a composition of a finite but arbitrary number of so-called operator layers. In this section we derive conditions on the components of an operator layer, such that it is a well-defined and continuous operator between two Lebesgue spaces. More precisely, for a bounded domain $\domain \subset \R^d$ and $1 \leq p,q \leq +\infty$ we aim to construct a continuous operator $\mathcal{L}:  L^p(\domain) \rightarrow L^q(\domain)$, such that an input function $u \in L^p(\domain)$ is mapped to 
\begin{align}\label{eq:layerdef}
    \mathcal{L}(u)(x) = \sigma\left(\Psi(u)(x)\right) \qquad \text{for \sk{a.e.} } x \in \domain,
\end{align} 
where we summarize all \sk{affine} operations with an operator $\Psi$, such that
\begin{align}\label{eq:linearpart}
    \Psi(u) = W u + \mathcal{K}u + b.
\end{align}
Here, the weighting by $W \in \R$ implements a residual component and the kernel integral operator $\mathcal{K}: u \mapsto \int_{\domain} \kappa(\cdot, y)\, u(y)\,dy$, determined by a kernel function $\kappa: \domain \times \domain \rightarrow \R$ generalizes the discrete weighting performed in neural networks. Analogously, the bias function $b: \domain \rightarrow \R$ is the continuous counterpart of a bias vector. The (non-linear) activation function $\sigma: \R \rightarrow \R$ is applied pointwise and thus acts as a Nemytskii operator (see e.g., \cite{goldberg92}).
Thus, with a slight abuse of notation, the associated Nemytskii operator takes the form%
\begin{align}\label{eq:nemytskii}%
\sigma: v \mapsto \sigma(v(\cdot)),
\end{align}
where we assume $\sigma$ to be a measurable function.
In order to ensure that the associated \Nem{} operator defines a mapping 
$\sigma:L^p(\domain)\to L^q(\domain)$ for 
$1\leq p,q \leq \infty$ we require 
the following conditions to hold:
\begin{align}\label{eq:cond}
\begin{aligned}
\underline{p,q<\infty}: \ &|\sigma(x)| \leq K +\beta |x|^{\frac{p}{q}} \text{ for all } x \in \R\text{ and constants } \beta, K\in\R,\\
\underline{p=\infty}: \ &|\sigma(x)| \leq K(c) \text{ for every } c>0 \text{ for all } x, \abs{x}<c\\
&\text{ and a constant }K(c)\in\R\text{ depending on }c,\\
\underline{p<\infty, q=\infty}: \ &|\sigma(x)| \leq K\text{ for all } x \in \R\text{ and a constant } K\in\R,
\end{aligned}
\end{align}
which were used in \cite{goldberg92}.
\begin{lemma}
For $1\leq p,q\leq \infty$ assume that $\sigma$ fulfills \eqref{eq:cond}. Then we have that the associated \Nem{} operator is a mapping $\sigma:L^p(\domain)\to L^q(\domain)$.
\end{lemma}
\begin{proof}
Similar to \cite[Th. 1]{goldberg92}, follows directly by employing the estimates in \eqref{eq:cond}.
\end{proof}
Since we are interested in continuity properties of the layer in \eqref{eq:layerdef} we consider the following continuity result for 
\Nem{} operators.
\begin{lemma}\label{lem:continuity}
For $1\leq p \leq \infty, 1\leq q < \infty$ assume that the function $\sigma:\R\to\R$ 
is continuous and uniformly continuous in the case $q=\infty$. If the associated 
\Nem{} operator is a mapping $\sigma: L^p(\domain)\to L^q(\domain)$ then it is continuous.
\end{lemma}
\begin{proof}
For $q<\infty$ the proof can be adapted from \cite[p. 155-158]{vauinberg74}. For the case $q=\infty$ we refer to \cite[Th. 5]{goldberg92}.
\end{proof}
\begin{remark}
For $1\leq p \leq q < \infty$
it is sufficient for $\sigma$ to be $p/q$-Hölder continuous or locally Lipschitz continuous for $p,q=\infty$. In that case the Hölder and respectively the Lipschitz continuity transfers to the \Nem{} operator, see \cite{tröltzsch}.
\end{remark}
\begin{example}\label{ex:relu}
The ReLU (Rectified Linear Unit, see \cite{fuku75}) function $\sigma(x) = \max(0,x)$ generates a continuous Nemytskii operator $\sigma: L^{p}(\Omega) \rightarrow L^{q}(\Omega)$ for any $p \geq q$. To show this, we note that the function $\sigma$ is Lipschitz-continuous and with $p \geq q$ we have for all $x\in\R$ that $
\abs{\sigma(x)} \leq \abs{x} \leq 1+ \abs{x}^{\frac{p}{q}}.$

\end{example}

\begin{proposition}
For $1 \leq p, q \leq \infty$ let $\mathcal{L}$ be an operator layer given by \eqref{eq:layerdef} with an activation function $\sigma: \R \rightarrow \R$.
If there exists $r \geq 1$ such that
\begin{enumerate}[label=(\roman*)]
    \item the \sk{affine} part defines a mapping $\Psi: L^p(\domain) \rightarrow L^r(\domain)$,
    \item the activation funtion $\sigma$ generates a Nemytskii operator $\sigma: L^r(\domain) \rightarrow L^q(\domain)$,
\end{enumerate}
then it holds that
$
\mathcal{L}: L^p(\domain) \rightarrow L^q(\domain).
$
If additionally $\Psi$ is a continuous operator on the specified spaces and the function $\sigma$ is continuous, or uniformly continuous in the case $q=\infty$, the operator $\mathcal{L}: L^p(\domain) \rightarrow L^q(\domain)$ is also continuous.
\end{proposition}
\begin{proof}
With the assumptions on $\sigma$ we directly have $\mathcal{L} = \sigma \circ \Psi:  L^p(\domain) \rightarrow L^q(\domain)$. The continuity of $\mathcal{L}$ follows from Lemma \ref{lem:continuity}.
\end{proof}
\begin{example}\label{ex:continuouslayer}
On the periodic domain $\domain = \R / \mathbb{Z}$ consider an \sk{affine} operator $\Psi$ as defined in \eqref{eq:linearpart}, where the integral operator is a convolution operator, i.e., $\kappa(x,y) = \kappa(x-y)$ with a slight abuse of notation. If for $1\leq p,r,s \leq \infty$ we have that $\kappa \in L^{s}(\domain)$ with 
$
{1}/{r} + 1 = {1}/{p} + {1}/{s},
$
it follows from Young's convolution inequality (see e.g., \cite[Th. 1.2.12]{grafakos}) that $\mathcal{K}: L^p(\domain) \rightarrow L^r(\domain)$ is continuous. 
If further $b \in L^r(\domain)$ and $W = 0$ in the case $r > p$, it follows directly that
$\Psi: L^p(\domain) \rightarrow L^{r}(\domain)$
is continuous.
\end{example}

\subsection{Differentiability}
 To analyze the differentiability of the neural operator layers we first transfer the result for general Nemytskii operators from \cite[Th. 7]{goldberg92} to our setting.%
\begin{theorem}\label{thm:frechetdiff}
Let $1\leq q < p < \infty$ or $q = p = \infty$ and $\sigma: \R \rightarrow \R$ a continuously differentiable function. Furthermore, let the Nemytskii operator associated to the derivative $\sigma'$ be a continuous operator
$
\sigma':L^{p}(\domain) \rightarrow L^{s}(\domain),
$
with coefficient 
$
s = {pq}/{(p-q)}
$
for $q < p$ and $s = \infty$ for $q = p = \infty$. Then, the Nemytskii operator associated to $\sigma$ is Fréchet-differentiable and its Fréchet-derivative 
$
D\sigma\sk{(v)}: L^{p}(\domain) \rightarrow L^q(\domain)
$
in $v \in L^p(\domain)$ is given by
\begin{align*}
D\sigma(v)(h) = \sigma'(v)\cdot h, \qquad\text{for all }h\in L^p(\domain).
\end{align*}
\end{theorem}
Since the ReLU activation function from Example \ref{ex:relu} is not differentiable, it does not fulfill the requirements of Theorem \ref{thm:frechetdiff}. An alternative is the so-called Gaussian Error Linear Unit (GELU), proposed in \cite{hen16}.
\begin{example}
    The GELU function $\sigma(x) = x\,\Phi(x)$, where $\Phi$ denotes the cumulative distribution function of the standard normal distribution, generates a Fréchet-differentiable Nemytskii operator with derivative \sk{$D\sigma(v):  L^p(\domain) \rightarrow L^{q}(\Omega)$ for any $p \geq q$ and $v\in L^p(\domain)$}.
    To show this, we compute 
    $
    \sigma'(x) = \Phi(x) + x \phi(x),
    $
    where $\phi(x) = \Phi'(x)$ is the standard normal distribution. We see that $\sigma'$ is continuous and further $
        |\sigma'(x)| \leq 1 + |x|/\sqrt{2\pi} \leq 1 + 1/\sqrt{2\pi} + |x|^{\frac{p}{q}}/\sqrt{2\pi}
    $ for all $p \geq q$.
\end{example}
\begin{proposition}\label{prop:frechet}
For $1 \leq p, q \leq \infty$, let $\mathcal{L}$ be an operator layer given by \eqref{eq:layerdef} with \sk{affine} part $\Psi$ as in \eqref{eq:linearpart}. If there exists $r>q$, or $r = q = \infty$ such that
\begin{enumerate}[label=(\roman*)]
    \item the \sk{affine} part is a continuous operator $\Psi: L^p(\domain) \rightarrow L^r(\domain)$,
    \item the activation function $\sigma: \R \rightarrow \R$ is continuously differentiable
    \item and the derivative of the activation function generates a Nemytskii operator $\sigma': L^r(\domain) \rightarrow \sk{ \left[L^r(\domain) \rightarrow L^{s}(\Omega) \right]}$ with $s = {rq}/{(r-q)}$,
\end{enumerate}
then it holds that $\mathcal{L}:  L^p(\domain) \rightarrow L^q(\domain)$ is Fréchet-differentiable \sk{in any $v \in L^p(\domain)$} with Fréchet-derivative $D\mathcal{L}\sk{(v)}: L^p(\domain) \rightarrow L^q(\Omega)$
\begin{align*}
    D\mathcal{L}(v)(h) = \sigma'(\Psi(v))\cdot\sk{\Tilde{\Psi}(h)},
\end{align*}
\sk{where $\Tilde{\Psi}$ denotes the linear part of $\Psi$, i.e., $\Tilde{\Psi} = \Psi - b$.}
\end{proposition}
\begin{proof}
\cref{thm:frechetdiff} yields that 
$D\sigma(v): L^r(\domain) \rightarrow L^q(\domain)$ is well defined and continuous for $v\in L^r(\Omega)$. Fréchet-differentiability of linear and continuous operators on Banach spaces (see e.g., \cite[Ex. 1.3]{ambrosetti}) yields the continuity of $\Psi: L^p(\domain) \rightarrow L^r(\domain)$ in all $v \in L^p(\domain)$ with $
D\Psi(v)(h) = \sk{\Tilde{\Psi}}(h).$
The claim follows from the chain-rule for Fréchet-differentiable operators, see \cite[Prop. 1.4 (ii)]{ambrosetti}.
\end{proof}
For $p < q$ Fréchet-differentiability of a Nemytskii operator implies that the generating function is constant, and respectively affine linear for $p=q<\infty$, see \cite[Ch 3.1]{goldberg92}. Therefore, unless $p = \infty$, Fréchet-differentiability of neural operators with non-affine linear activation functions is only achieved at the cost of mapping \sk{the output of the affine part} into a less regular space.
\begin{example}\label{ex:fno}
    For a continuous convolutional neural operator layer as constructed in Example \ref{ex:continuouslayer}, we consider a parametrization of the kernel function by a set of parameters $\hat{\theta} = \lbrace\hat{\theta}_k\rbrace_{k \in I} \subset \mathbb{C}$, where $I$ is a finite set of indices, such that
    \begin{align}\label{eq:fnokernel}
        \kappa_{\hat{\theta}}(x) = \sum_{k \in I} \hat{\theta}_k\, b_k(x),
    \end{align}
    with Fourier basis functions $b_k(x) = \exp{(2\pi i\, kx)}$ \sk{ for $x \in \domain$}.
    Effectively, this amounts to parametrizing the kernel function by a finite number of Fourier coefficients. The resulting linear operator and the operator layer are denoted by $\Psi_{\hat{\theta}}$ and $\mathcal{L}_{\hat{\theta}}$. We note that FNOs proposed in \cite{li2021fourier} are neural operators that consist of such layers. It is easily seen that the kernel function defined by \eqref{eq:fnokernel} is bounded and thus $\kappa_{\hat{\theta}} \in L^{\infty}(\domain)$. Therefore, for suitable activation functions, Proposition \ref{prop:frechet} yields Fréchet-differentiability of $\mathcal{L}_{\hat{\theta}}$ with respect to its input function $v$, \sk{which was similarly observed in \cite{li2021physics}}. Additionally, for fixed $v$ we consider the operator $\mathcal{L}_{(\cdot)}(v):\hat{\theta} \mapsto\ \mathcal{L}_{\hat{\theta}}(v)$ which maps a set of parameters to a function. With the arguments from \cref{prop:frechet} we derive the partial Fréchet-derivatives of an FNO-layer with respect to its parameters for $h_k=(1+i)\ e_k$ as
    \begin{align*}
        D_{\hat{\theta}_k}\mathcal{L}_{\hat{\theta}}(v) := D\mathcal{L}_{\hat{\theta}}(v)(h_k) = \sigma'(\Psi_{\sk{\hat{\theta}}}(v))\, \sk{D\Psi}\sk{_{\hat{\theta}}(v)(h_k)},
    \end{align*}
    where $e_k$ denotes the \sk{$k$-th} canonical basis vector. Computing 
    the Fréchet-derivative of $\Psi$ \sk{in the sense of Wirtinger calculus (\cite[Ch. 1]{Remmert1991}), this can be rewritten as} $D_{\hat{\theta}_k}\mathcal{L}_{\hat{\theta}}(v) = \sigma'(\Psi_{\sk{\hat{\theta}}}(v))\, \bar{\hat{v}}_k\,\bar{b}_k$, where $\hat{v}_k$ denotes the $k$-th Fourier coefficient of $v$. Here, for a 
    complex number $z\in\mathbb{C}$, we denote by $\bar{z}$ its complex conjugate.
\end{example}
\section{Connections to Convolutional Neural Networks}\label{sec:cnnfno}
In this section we analyze the connection between FNOs and CNNs. Thus, for the remainder of this work, we set the domain to be the $d$-dimensional torus, i.e., $\domain = \mathbb{R}^d / \mathbb{Z}^d$. As described in \cref{ex:fno}, the main idea of FNOs is to parametrize the convolution kernel by a finite number of Fourier coefficients $\hat{\theta} = \lbrace\hat{\theta}_k\rbrace_{k \in I} \subset \mathbb{C}$, where $I \subset \mathbb{Z}^d$ is a finite set of indices. Making use of the convolution theorem, see e.g., \cite[Prop. 3.1.2 (9)]{grafakos}, the kernel integral operator can then be written as
\begin{align}\label{eq:fnolayer}
    \mathcal{K}_{\hat{\theta}} v = \mathcal{F}^{-1}\left(\hat{\theta} \cdot \mathcal{F}v\right),
\end{align}
where $\mathcal{F}\sk{: \left[\Omega \rightarrow \mathbb{C}\right] \rightarrow \left[\mathbb{Z}^d \rightarrow  \mathbb{C}\right]}$ denotes the Fourier transform \sk{on the torus (see e.g., \cite[Ch. 3]{grafakos})} and $\cdot$ denotes elementwise multiplication in the sense that
\begin{align}\label{eq:fnocontzero}
    \left(\hat{\theta} \cdot \mathcal{F}v\right)_k = 
        \begin{cases}
        \hat{\theta}_k\, \left(\mathcal{F}v\right)_k & \text{for } k \in I,\\
        0 & \text{otherwise.}
    \end{cases}
\end{align} 
We only consider parameters such that $\mathcal{K}$ maps real-valued functions to real-valued functions. This is equivalent to Hermitian symmetry, i.e., $\hat{\theta}_k = \overline{\hat{\theta}_{-k}}$ and in particular $\hat{\theta}_0 \in \R$. As proposed in \cite{li2021fourier}, for $N \in \mathbb{N}$ we choose the set of \sk{multi-}indices
    $I_N := \lbrace -\lceil (N-1)/2\rceil, \,\hdots\,,0,\, \hdots\,, \lfloor (N-1)/2 \rfloor\rbrace^d$, which corresponds to parametrizing the $N$ lowest frequencies in each dimension. This is in accordance to the universal approximation result for FNOs derived in \cite{kovachki2021universal}. At this point, we assume $N$ to be an odd number to avoid problems with the required symmetry and expand the approach to even choices of $N$ in Section \ref{subsec:evendims}.
Although an FNO is represented by a finite number of parameters, a discretization of \eqref{eq:fnolayer} is needed to process discrete data, e.g., digital images. We therefore define the set of spatial \sk{multi-}indices $J_N := \left\lbrace 0,\,\hdots\,,N-1 \right\rbrace^d$ and write $v \in \R^{J_N}$ for mappings $v: J_N \rightarrow \R$. Furthermore, we discretize the Fourier transform for $v \in \R^{J_N}$ as
\begin{align*}
    \left(Fv\right)_k = \frac{1}{\lambda} \sum_{j \in J_N} v_j\, e^{-2\pi i \, \left\langle k, \frac{j}{N}\right\rangle} \qquad \text{for all } k \in I_N
\end{align*}
and its inverse for $\hat{v} \in \mathbb{C}^{I_N}$ as
\begin{align*}
\left(F^{-1}\hat{v}\right)_j = \frac{\lambda}{|J_N|}\sum_{k \in I_N} \hat{v}_k\, e^{2\pi i \, \left\langle k, \frac{j}{N}\right\rangle} \;\;\quad \quad\quad\text{for all } j \in J_N,
\end{align*}
where $\lambda \in \{1, \sqrt{|J_N|}, |J_N|\}$ determines the normalization factor.
The discretized convolution operator parametrized by $\hat{\theta} \in \mathbb{C}^{I_N}_{\text{sym}} := F(\R^{J_N})$ is then defined by
\begin{align*}\label{eq:discfno}
    K(\hat{\theta})(v) = F^{-1}\left(\hat{\theta} \cdot Fv\right) \qquad \text{for }v \in \R^{J_N}.
\end{align*}
For the remainder of this work, we refer to the above implementation of convolution as the FNO-implementation.
In the following we compare the FNO-implementation to the standard implementation of the convolution of $\theta$ and $v \in \R^{J_N}$ in a conventional CNN, which can be expressed as
\begin{align*}
    C(\theta)(v)_j = \sum_{\Tilde{\jmath}\in J_N} \theta_{j-\Tilde{\jmath}}\,v_{\Tilde{\jmath}} \qquad \text{for all } j \in J_N.
\end{align*}
For the sake of simplicity, we handle negative indices by assuming that the values can be perpetuated periodically, although this is usually not done in practice.
\subsection{Extension to Higher Input-Dimensions by Zero-Padding}
So far, the presented implementations of convolution require the dimensions of the parameters $\theta$, or $\hat{\theta}$ and the input $v$ to coincide. In accordance to \eqref{eq:fnocontzero}, the authors of \cite{li2021fourier} propose to handle dimension mismatches by zero-padding of the spectral parameters. More precisely, a low-dimensional set of parameters $\hat{\theta} \in \mathbb{C}^{I_M}$ is adapted to an input $v \in  \R^{J_{N}}$ with odd $N \in \mathbb{N}$ by setting
\begin{align*}
    \hat{\theta}^{M \rightarrow N}_k = \begin{cases}
        \hat{\theta}_k & \text{for } k \in I_N \cap I_M,\\
        0 & \text{for } k \in I_N \backslash I_M.
    \end{cases}
\end{align*}
Since we choose $N$ to be odd, the required symmetry is not hurt by the above operation.  
The extended FNO-implementation of the convolution is then given for $\hat{\theta} \in \mathbb{C}^{I_M}$ and $v \in  \R^{J_{N}}$ by
\begin{equation*}
    K(\hat{\theta})(v) := K(\hat{\theta}^{M \rightarrow N})(v).
\end{equation*}
Analogously, in the conventional CNN-implementation the convolution of parameters $\theta \in \R^{J_M}$ and $v \in \R^{J_N}$ with $N \geq M$ is computed as 
\begin{equation*}
    C(\theta)(v) := C(\theta^{M\rightarrow N})(v),
\end{equation*}
where again, $\theta^{M\rightarrow N} \in \R^{J_M}$ denotes the zero-padded version of $\theta$.
We stress that, although the technique to generalize the implementations to higher input dimensions is the same, the outcome differs substantially. This was already mentioned in \cite[Sec. 4]{kovachki2021neural} and is discussed further in Section \ref{subsec:resinv}.
\subsection{Convertibility and Complexity}
Deriving FNOs from convolutional neural operators using the convolution theorem suggests that there is a way to convert one implementation of convolution into the other \sk{as long as the input dimension is fixed}. The following Lemma shows that this is indeed possible.
\begin{lemma}\label{lem:convertibility}
    Let $M \leq N$ both be odd and let $T: \R^{J_N} \rightarrow \mathbb{C}^{I_N}$ be defined for $\theta \in \R^{J_N}$ as
        $T(\theta) = \lambda \, F(\theta)$.
    For any $\theta \in \R^{J_M}$ and $v \in \R^{J_N}$ it holds true that
    \begin{align*}
        C(\theta)(v) = K(T(\theta^{M\rightarrow N}))(v)
    \end{align*}
    and for any $\hat{\theta} \in \mathbb{C}^{I_M}_{\text{sym}}$ and $v \in \R^{J_N}$ it holds true that
    \begin{align*}
         K(\hat{\theta})(v) = C(T^{-1}(\hat{\theta}^{M\rightarrow N}))(v).
    \end{align*}
\end{lemma}
\begin{proof}
    By the definition of the extension to higher input dimensions we can assume $M = N$. For $\theta, v \in \R^{J_N}$ we derive the discrete analogon of the convolution theorem by inserting the definitions of the discrete Fourier transform as
    \begin{align*}
        F(C(\theta)(v)) = \lambda\,F(\theta)\cdot F(v).
    \end{align*}
    Employing that $F: \R^{J_N} \rightarrow \mathbb{C}^{I_N}_{\text{sym}}$ is a bijection, it follows that
    \begin{align*}
        C(\theta)(v) = F^{-1}(\lambda\,F(\theta)\,F(v)) = K(T(\theta))(v).
    \end{align*}
    The second statement can be proven analogously. We note that $T^{-1}$ is well-defined since $\lambda \geq 1$ for odd $N$.
\end{proof}
Although the above Lemma proves convertibility for a fixed set of parameters and fixed input dimensions, a conversion can increase the amount of required parameters, as in general, the dimension of the converted parameters has to match the input dimension. It becomes clear that spatial locality cannot be enforced with the proposed FNO-parametrization and spectral locality cannot be enforced with the CNN-parametrization. Therefore, different behavior during the training process is to be expected if the parameter size does not match the input size. Moreover, the following Lemma shows that even for matching dimensions, equivalent behavior for gradient-based optimization like steepest descent requires careful adaptation of the learning rate, since the computation of gradients is not equivariant with respect to the function $T$.  
\begin{lemma}
    For odd $N \in \mathbb{N}$ and $v, \theta \in \R^{J_N}$ and $\hat{\theta} = T(\theta)$ it holds true that
    \begin{align*}
        \nabla_{\hat{\theta}} K(\hat{\theta})(v) = \frac{1}{|J_N|}\; T \left( \vphantom{\hat{\theta}} \nabla_{\theta}C(\theta)(v) \right). 
    \end{align*}
\end{lemma}
\begin{proof}
Inserting $\hat{\theta} = T(\theta)$ it follows with the chain rule from Lemma \ref{lem:convertibility} that
    \begin{align*}
        \frac{\partial K(\hat{\theta})(v)_l}{\partial \hat{\theta}_k}= \sum_{j \in J_N} \frac{\partial C(\theta)(v)_l}{\partial \theta_j}\, \frac{\partial\sk{ T^{-1}(\hat{\theta})}_j}{\partial \hat{\theta}_k} = \sk{\frac{1}{|J_N|} }\sum_{j \in J_N} \frac{\partial C(\theta)(v)_l}{\partial \theta_j}\,e^{-i2\pi \langle k, \frac{j}{N}\rangle}
    \end{align*}
    for $k \in I_N$.
The claim now follows by inserting the definition of $T$.
\end{proof}
\subsection{Adaptation to Even Dimensions}\label{subsec:evendims}
\begin{figure}[t]%
\centering%
\includegraphics[width=.6\textwidth]{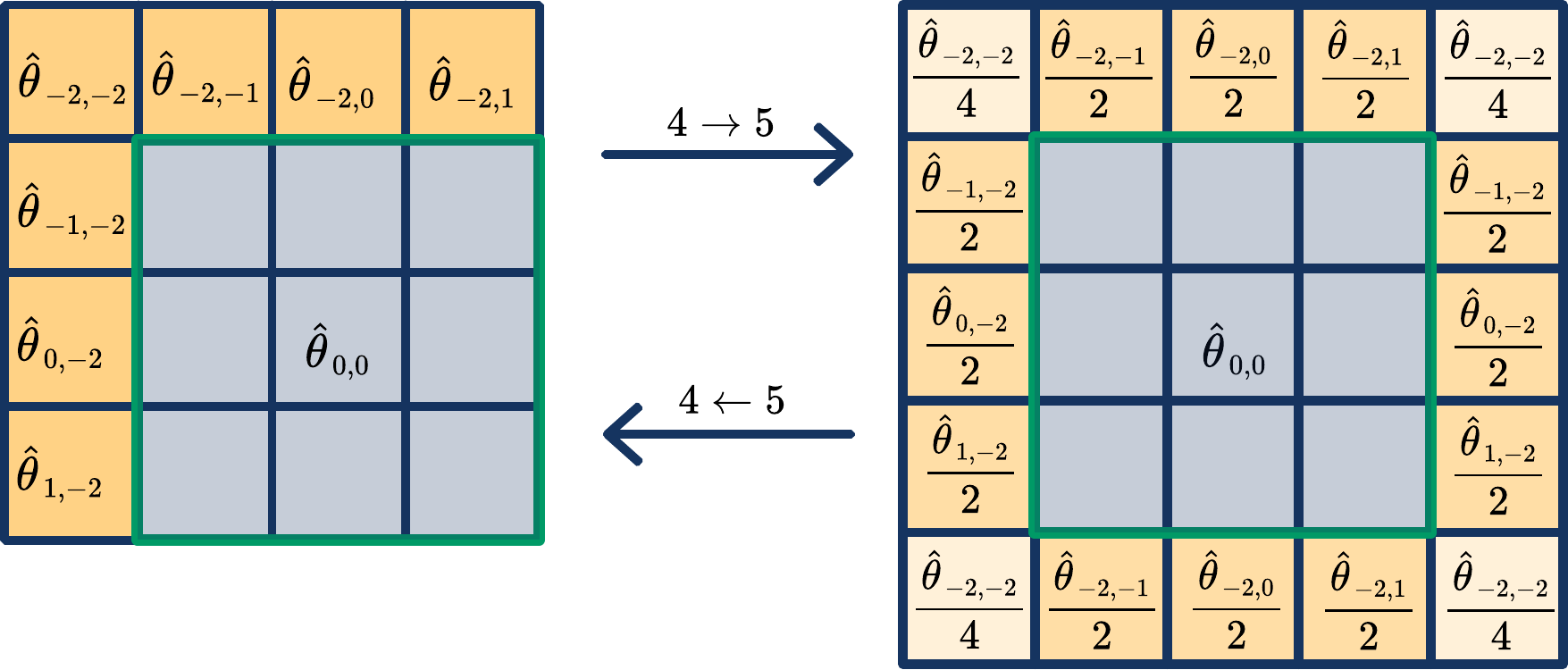}
\caption{Nyquist splitting for spectral parameters to extend real-valued trigonometric interpolation to even dimensions.}%
\label{fig:nyquist}%
\end{figure}
For the remainder of this paper we consider the special case $\domain = \R^2 / \mathbb{Z}^2$ and adapt the FNO-implementation to even dimensions.
For odd dimensions $M, N$, zero-padding of a set of spectral coefficients does not violate the requirement $\hat{\theta}^{M \to N} \in \mathbb{C}^{I_N}_{\text{sym}}$. This property is lost in general for even dimensions. Since for odd dimensions, zero-padding in the spectral domain is equivalent to trigonometric interpolation, we perform the adaptation of dimensions such that $\hat{\theta}^{M \rightarrow N}$ is a trigonometric interpolator of a real-valued function \sk{(see \cite{trigo19} for an exhaustive study on this topic)}. In practice, this means splitting the coefficients corresponding to the Nyquist frequencies to interpolate from an even dimension to the next higher odd dimension, or to invert this splitting to interpolate from an odd dimension to the next lower even dimension (see Figure \ref{fig:nyquist}).  
The real-valued trigonometric interpolation of $v \in \R^{J_M}$ to a dimension $N$ is then given by
\begin{align*}
    v^{M \xrightarrow{\Delta} N} := F^{-1} \left( (Fv)^{M \rightarrow N}\right).
\end{align*}
We extend the FNO-implementation to parameters $\hat{\theta} \in \mathbb{C}^{I_N}_{\text{sym}}$ and inputs $v \in \R^{J_N}$ with even $N \in \mathbb{N}$ by defining
\begin{align}\label{eq:evendims}
    K(\hat{\theta})(v) := \left(K(\hat{\theta}^{N \rightarrow \Tilde{N}})(v^{N\xrightarrow{\Delta}\Tilde{N}})\right)^{\Tilde{N}\xrightarrow{\Delta} N},
\end{align}
where $\Tilde{N} = N+1$. 
We note that by this choice we lose the direct convertibility to the CNN-implementation as in general for even dimensions
\begin{align*}
    K(\hat{\theta})(v) \neq F^{-1}(T^{-1}(\hat{\theta})\, Fv),
\end{align*}
as the right hand side corresponds to zero-padding of the spectral coefficients regardless of the oddity of the dimensions.
However, we can still convert the FNO-implementation to the CNN-implementation and vice versa, by adapting the magnitude of coefficients to the effects of the Nyquist splitting.
\subsection{Interpolation Equivariance}\label{subsec:resinv}
Our motivation to perform the adaptation to even dimension as proposed in the preceding section, is that the resulting implementation of convolution is equivariant with respect to (real-valued) trigonometric interpolation. 
\begin{corollary}
For $\hat{\theta} \in \mathbb{C}^{I_M}_{\text{sym}}$, $v \in \R^{J_N}, M \leq N$ it holds true for any $L \geq M$ that
    \begin{align*}
        K(\hat{\theta})(v^{N \xrightarrow{\Delta} L}) = \left( K(\hat{\theta})(v) \right)^{N \xrightarrow{\Delta} L}.
    \end{align*}
\end{corollary}
\begin{proof}
    We first note that it holds for any choice of $M \leq N,L$ that
    \begin{equation*}
        K(\hat{\theta})(v^{N \xrightarrow{\Delta} L})  = \left(K(\hat{\theta}^{M \to \Tilde{M} \to \Tilde{L}})(v^{N \xrightarrow{\Delta} \Tilde{N} \xrightarrow{\Delta} \Tilde{L}})\right)^{\Tilde{L} \xrightarrow{\Delta} L},
    \end{equation*}
    where $\Tilde{L} = L+(1-L\%2)$, $\Tilde{M} = M+(1-M\%2)$, $\Tilde{N} = N+(1-N\%2)$ \sk{and $\%$ denotes the modulo operation.} Therefore, we can assume $L,M$ and $N$ to be odd without loss of generality and thus $\Tilde{L} = L, \Tilde{M} = M$ and $\Tilde{N} = N$. Regarding the discrete Fourier coefficients then reveals that
    \begin{align*}
        &F(K(\hat{\theta})(v^{N \xrightarrow{\Delta}L}))_k =
        %
        \left\{%
        \begin{aligned}
            \hat{\theta}_k\,F(v)_k & \quad\text{ for } k \in I_M,\\
            0 & \quad\text{ otherwise}
        \end{aligned}
        \right\}
        = F((K(\hat{\theta})(v))^{N \xrightarrow{\Delta}L})_k.
    \end{align*}
    Applying the inverse Fourier transform completes the proof.
\end{proof}
\section{Numerical Examples}\label{sec:num}
\sk{In this section we compare the discussed implementations of convolution numerically in the context of image classification.\footnote{Our code is available online:  \href{https://github.com/samirak98/FourierImaging}{github.com/samirak98/FourierImaging}.} Here, the task is to assign a label from $s\in\mathbb{N}$ possible classes to a given image $v: [0,1]^2 \rightarrow \R^{n_c}$, with $n_c \in \mathbb{N}$ denoting the number of color channels. Solving this task numerically requires discrete input images of the form $v^{N} = v|_{J_N/N} \in \R^{J_N \times n_c}$, where $N \in \mathbb{N}$ denotes the dimension. We note that since we consider a fixed function domain the dimension is proportional to the resolution. If we assume $N$ to be fixed the network is a function $f_\theta:\R^{J_N \times n_c}\to\R^s$. Given a finite training set $D\subset \R^{J_N\times nc} \times \R^s$ we optimize the parameters $\theta$ by minimizing the empirical loss based on the cross-entropy \cite[Ch. 3]{Goodfellow16}.
The networks we use for our experiments consist of several convolutional layers for feature extraction followed by one fully connected classification layer. To make all architectures applicable to inputs of any resolution, we insert an adaptive average pooling layer between the feature extractor and the classifier. 
}
\subsection{Expressivity for Varying Kernel Sizes}
In the first experiment (see \cref{fig:ksize}) we train a CNN \sk{without any residual components} on the FashionMNIST\footnote{This dataset consists of $60,000$ training and $10,000$ test $28\times 28$ images (grayscale).} dataset. The network has two convolutional layers with periodic padding and without striding, followed by an adaptive pooling layer and a linear classifier. Since we do not observe major performance changes on the test set for different kernel sizes, we conclude that on this data set the expressivity of the small kernel architectures is comparable to large kernel architectures. We then convert the convolutional layers of the CNNs with $3\times3$- and $28 \times 28$-kernels to FNO-layers, employing varying numbers of spectral parameters. Here, we observe decreasing performance with smaller spectral kernel sizes, indicating that the learned spatial kernels cannot be expressed well by fewer frequencies. However, in this example, training an FNO with the same structure almost closes this performance gap. This implies the existence of low frequency kernels with sufficient expressivity. \sk{We refer to \cite{williamson22} for a study on training with a spectral parametrization.}
\begin{figure}[t]
\centering
\includegraphics[width=.5\textwidth]{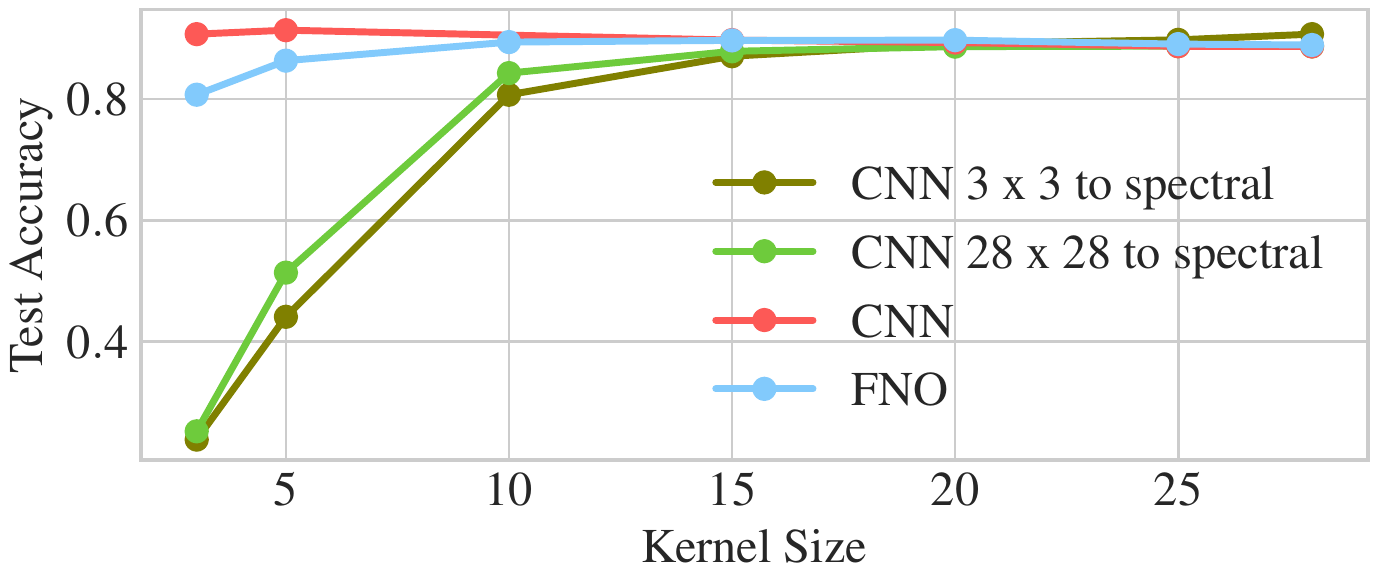}
\caption{Test accuracy of CNNs and FNOs for varying kernel sizes.}
\label{fig:ksize}
\end{figure}
\subsection{Resolution Invariance}%
In the second experiment, we investigate the resolution invariance of the different convolution implementations. In \cref{fig:performancea} we compare the accuracy on test data resized to different resolutions with trigonometric, or bilinear interpolation, respectively. Here, CNN refers to the conventional CNN-implementation with $5 \times 5$ kernel, where dimension mismatches are compensated for by spatial zero-padding of the kernel. FNO refers to the FNO-implementation, where the kernels are adapted to the input dimension by trigonometric interpolation. Additionally, we show the behavior of the CNN for inputs rescaled to the training resolution. Applying trigonometric interpolation before a convolutional layer can be interpreted as an FNO-layer with predetermined output dimensions.

The performance of the CNN varies drastically with the input dimension and peaks for the resolution it was trained on. This result is in accordance with the effect showcased in \cref{fig:bulbul}: Dimension adaption via spatial zero-padding modifies the locality of the kernel and consequently captures different features for different resolutions. While trigonometric interpolation performs best, we see that the FNO adapts very well. In particular, the performance for higher input resolutions deters only slightly, which is not the case for the standard CNN.

Additionally (see \cref{fig:performanceb}), we train a ResNet18 \cite{he2016deep} on the Birds500 data set\footnote{\sk{We employ a former version of the data set, which} consists of $76,262$ RGB images for training and $2,250$ images for testing of size $224\times 224$, where the task is to classify birds out of $450$ possible classes.} with a reduced training size of $112 \times 112$. To regularize the generalization to different resolutions, especially for the FNO-implementation, we replace the standard striding operations by trigonometric downsampling. Compared to the first experiment it stands out that the FNO performs worse for inputs with resolutions below $112 \times 112$, but only slightly diminishes for higher resolutions. We attribute this fact to the dimension reduction operations in the architecture. 
\begin{figure}[t]
\begin{subfigure}{.5\textwidth}
\includegraphics[width=\textwidth]{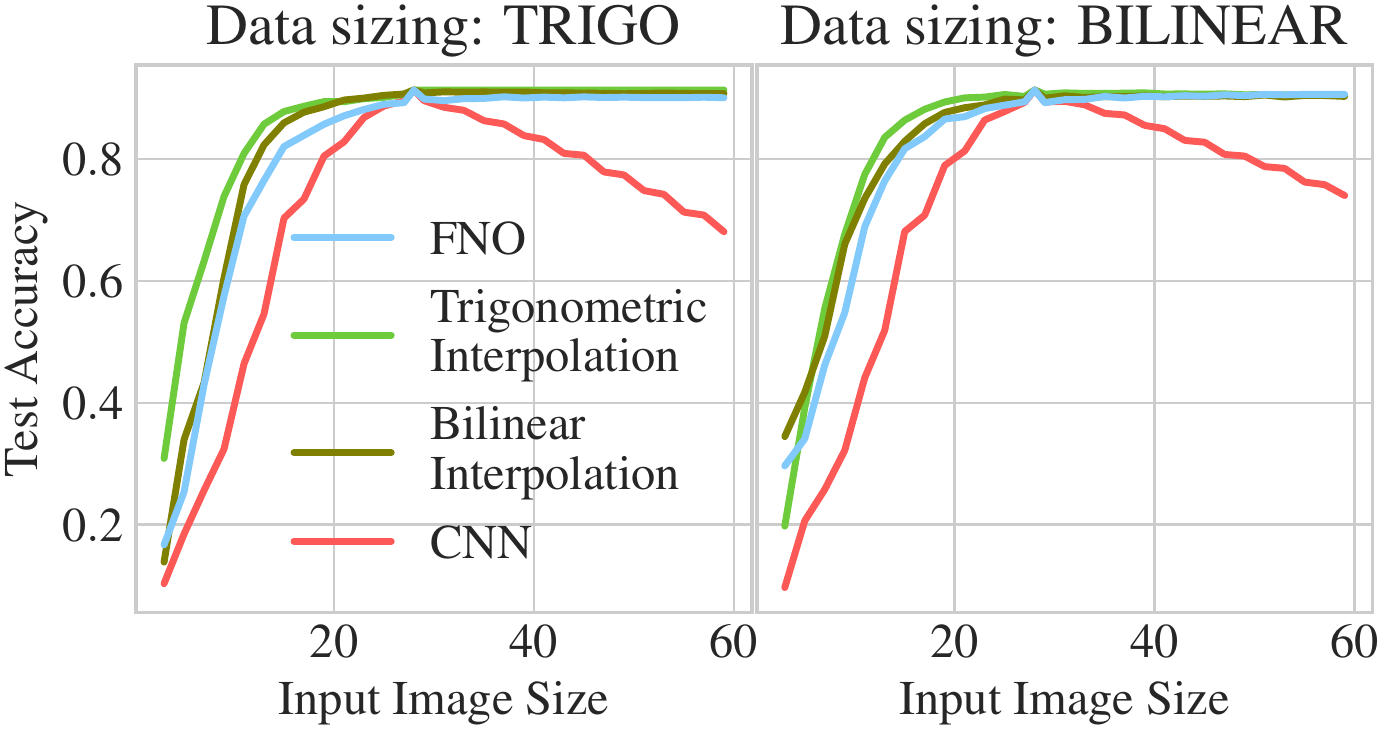}
\caption{FMNIST}\label{fig:performancea}
\end{subfigure}
\begin{subfigure}{.5\textwidth}
\includegraphics[width=\textwidth]{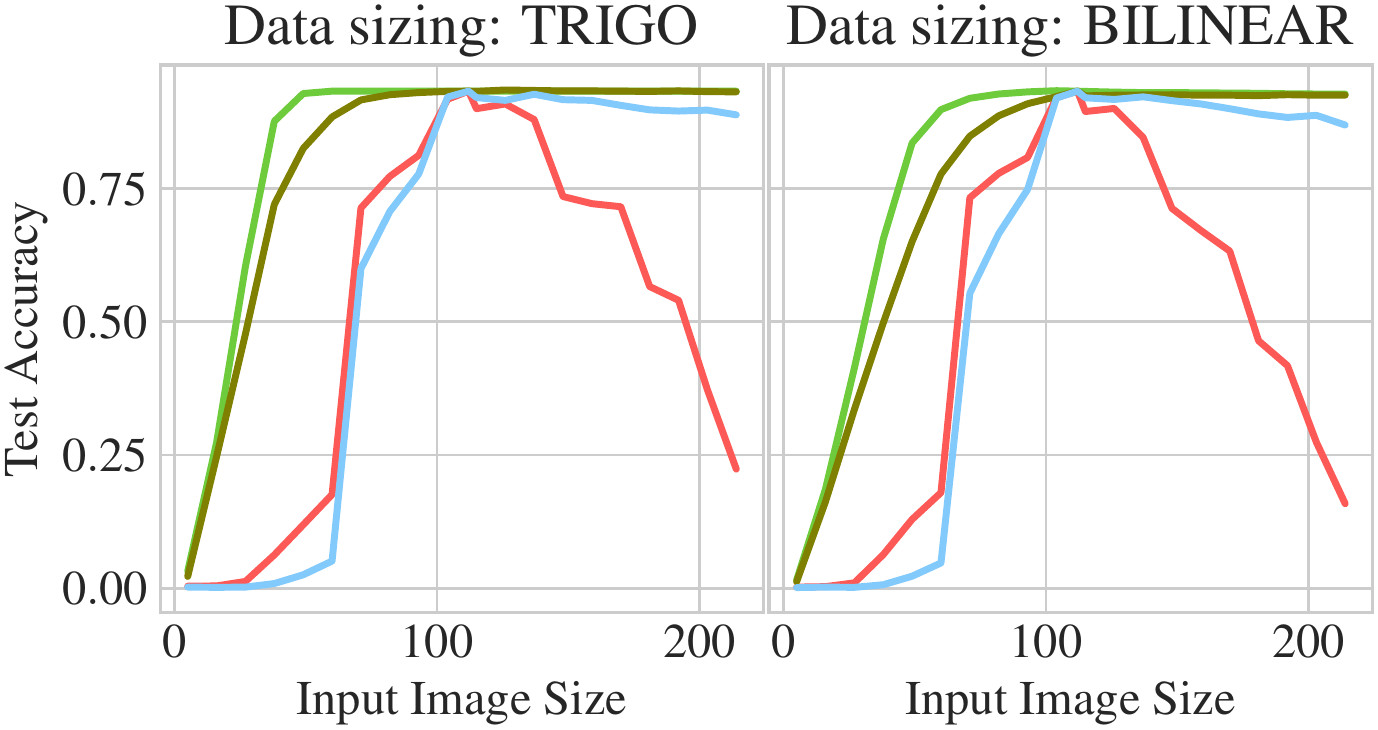}
\caption{Birds500}\label{fig:performanceb}
\end{subfigure}
\caption{Performance for different interpolation methods on test data that has been resized with the interpolation method denoted on top of the plots.}
\end{figure}
\section{Conclusion and Outlook}
In this work, we have studied the regularity of neural operators on Lebesgue spaces and investigated the effects of implementing convolutional layers in the sense of FNOs. Based on the theoretical derivation of the convertibility from standard CNNs to FNOs, our numerical experiments show that it is possible to convert a network that was trained with the standard CNN architecture into an FNO. By this, we could combine the benefits of both approaches: Enforced spatial locality with a small number of parameters during training and an implementation that generalizes well to higher input dimensions during the evaluation. However, we have seen that the trigonometric interpolation of inputs outperforms all other considered approaches. In future work, we want to investigate how the ideas of FNOs and trigonometric interpolation can be incorporated into image-to-image architectures like U-Nets as proposed in \cite{ronneberger15}. Additionally, we want to further explore the effects of training in the spectral domain, for example with respect to adversarial robustness.
\bibliographystyle{splncs04}
\bibliography{bibliography}
\end{document}